\documentclass[runningheads]{llncs}
\usepackage[T1]{fontenc}
\usepackage{graphicx}
\usepackage{amsmath}
\usepackage{amsfonts}
\usepackage{mathtools}
\usepackage{cancel}
\usepackage{float} 
\usepackage{algorithm}
\usepackage{overpic}
\usepackage[noend]{algpseudocode}
\usepackage{amssymb}

\usepackage{xcolor}
\usepackage{xspace}
\usepackage{comment}

\usepackage[a-2b,mathxmp]{pdfx}[2018/12/22]

\def\gstp{\textsc{GSTP}\xspace}
\def\cgsp{\textsc{CGSP}\xspace}
\def\bgsp{\textsc{BGSP}\xspace}
\def\pgsp{\textsc{PGSP}\xspace}
\def\mapf{\textsc{MAPF}\xspace}

\def\mocgsp{{\textsc{MOCGSP}}\xspace}
\def\mobgsp{{\textsc{MOBGSP}}\xspace}
\def\mopgsp{{\textsc{MOPGSP}}\xspace}

\def\scg{\textsc{SCG}\xspace}

\def\E{\mathbb{E}}
\DeclareMathOperator{\Var}{Var}
\DeclareMathOperator{\Cov}{Cov}

\usepackage{xcolor}
\usepackage{xargs} 
\usepackage[textsize=footnotesize]{todonotes}

\newcommandx{\mg}[2][1=]{}
\newcommandx{\jy}[2][1=]{}
\begin{document}
\title{Optimally Solving Colored Generalized Sliding-Tile Puzzles: Complexity and Bounds}
\titlerunning{Colored Generalized Sliding-Tile Puzzles}
\author{Marcus Gozon \inst{1} %\orcidID{0000-1111-2222-3333} 
\and Jingjin Yu \inst{2}\orcidID{0000-0003-4112-2250}}
\authorrunning{M. Gozon and J. Yu}
\institute{University of Michigan, Ann Arbor, MI 48109, USA \\ \email{mgozon@umich.edu} \and
Rutgers University, New Brunswick, NJ 08901, USA \\
\email{jingjin.yu@cs.rutgers.edu}}
\maketitle              % typeset the header of the contribution
\begin{abstract}
\;\;\;\;The generalized sliding-tile puzzle (\textsc{GSTP}), allowing many\\ square tiles on a board to move in parallel while enforcing natural geometric collision constraints on the movement of neighboring tiles, provide a high-fidelity mathematical model for many high-utility existing and future multi-robot applications, e.g., at mobile robot-based warehouses or autonomous garages.
Motivated by practical relevance, this work examines a further generalization of \gstp called the \emph{colored generalized sliding-tile puzzle} (\cgsp), where tiles can now assume varying degrees of distinguishability, a common occurrence in the aforementioned applications. Our study establishes the computational complexity of \cgsp and its key sub-problems under a broad spectrum of possible conditions and characterizes solution makespan lower and upper bounds that differ by at most a logarithmic factor.
These results are further extended to higher-dimensional versions of the puzzle game.

\end{abstract}

\section{Introduction}
% Start the introduction with potential applications and then introduce the problem. Describe what we have done in the previous work and why this new work is interesting (previous work is a special, extreme case)
Sliding-tile puzzles, such as the 15-puzzle \cite{Wikipedia15PUZ,loyd1959mathematical} and Klotski \cite{WikiKlotski} (see., e.g., Fig.~\ref{fig:15-gstp}), ask a player to sequentially move interlocked tiles on a board via one or more \emph{escorts}, or swap spaces, to reach some desired goal configurations. Such problems, modeling a wide array of important real-world challenges/applications including at autonomous mobile robot-based warehouses \cite{wurman2008coordinating,mason2019developing} and in garage automation \cite{yunfeng2022efficient,guo2023toward}, have inspired and/or served as the base model in many complexity and algorithmic studies, e.g.,  \cite{wilson1974graph,RATNER1990111,auletta1999linear,goldreich2011finding,kornhauser1984coordinating,goraly2010multi}, especially drawing attention from the combinatorial search community \cite{culberson1994efficiently,culberson1998pattern,sharon2015conflict,li2021eecbs,okumura2023lacam}. 

\begin{figure}
    \centering
    \includegraphics[width=0.8\columnwidth]{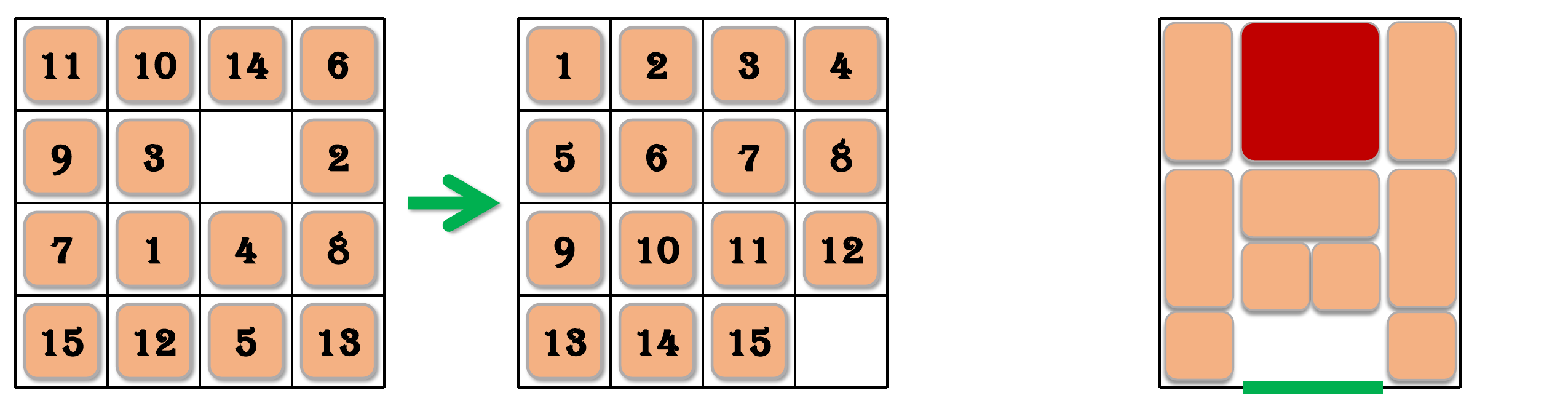}
    %{figs/MOGSTP Outline assembled.png}
    %\put(11.3, 2){{\small (a)}}
    %\put(67.8, 2){{\small (b)}}
    \caption{[left] Start and goal configurations of a $15$-puzzle instance. In the generalized sliding-tile puzzle, modeled after the $15$-puzzle, there can be $1+$ escorts and multiple tiles may move synchronously, e.g., tile $3$ and $9$ may move to the right in a single step in the left configuration. [right] In a Klotski puzzle (bearing many other names such as Huarong Road) \cite{WikiKlotski}, tiles are rearranged via sliding to allow the large square red title to ``escape'' from the green opening. Whereas the goal configuration in the $15$-puzzle is a single fixed one, the goal configuration in the Klotski only specifies the location of the largest tile and leaves the other tiles' final configuration unspecified.}
    \label{fig:15-gstp}
\end{figure}

Inspired by these previous studies, the vast application potential, and observing the possibility to simultaneously move multiple tiles even when there is a single escort in many real-world settings, we recently examined the \emph{generalized sliding-tile puzzle} (\gstp) \cite{gozon2024computing}.
In a \gstp instance, on an $m_1 \times m_2$ board lie $n < m_1m_2$ \emph{labeled} square tiles. Using the leftover $m_1m_2 - n$ escorts as swap spaces, tiles must be rearranged into a goal configuration. It is established that minimizing the required makespan for \gstp is NP-hard and that the fundamental optimality lower bounds match the polynomial-time algorithmic upper bounds. A key difference between \gstp and most well-studied \emph{multi-robot path planning} (MRPP) problems \cite{luna2011push,yu2013multi,YuLav16TOR,yu2018constant,demaine2019coordinated} or the largely equivalent \emph{multi-agent path-finding} (\mapf) problems \cite{silver2005cooperative,sharon2015conflict,stern2019multi,li2021eecbs,okumura2023lacam} is that \gstp enforces the \emph{corner following constraint} (CFC). The CFC forbids two neighboring tiles located at coordinates $(x, y)$ and $(x + 1, y)$ from executing the synchronous move where the first tile moves to $(x, y \pm 1)$ while the second tile moves to $(x, y)$ simultaneously. The CFC is a common physical constraint in practical warehouse and garage automation applications, e.g., two bordering rectangular robots must abide by the CFC to avoid collision.

In this work, we investigate a natural generalization of \gstp where tiles do not necessarily have unique labels. Instead, tiles can take some $k \ge 2$ colors where all tiles of the same color are indistinguishable/interchangeable, so swapping two tiles of the same color does not create a new tile configuration.
We call this problem the \emph{colored generalized sliding-tile puzzle} (\cgsp).
Our effort in this paper is mainly devoted to two key sub-problems of \cgsp. In the first, the number of colors is fixed at $k = 2$, yielding the \emph{binary generalized sliding-tile puzzle} (\bgsp). 
In the second, among a total of $n \ge k$ tiles, tiles $1$ to $k-1$ each have unique colors $1$ to $k-1$, respectively, while tiles $k$ to $n$ have the same color $k$. We denote this \cgsp variant as the \emph{partially-colored generalized sliding-tile puzzle} (\pgsp).

As a summary of our main contributions, we make a rich classification of the complexity and makespan bounds of optimally solving \bgsp, \pgsp, and \cgsp, covering all possible numbers of escorts with further generalizations to higher dimensions. 
On the side of the computational complexity, we show that it is NP-hard to compute makespan optimal solutions for \bgsp and \pgsp (and therefore, \cgsp) even when there is a single escort.
The hardness result generalizes to an arbitrary number of escorts and to higher-dimensional grid settings. In contrast, the hardness results for \gstp are only established for a large number of escorts \cite{gozon2024computing}.
On the side of achievable makespan optimality, the lower/upper bounds for \bgsp and \pgsp are similar for different numbers of escorts and problem parameters. The bounds for \bgsp, differing up to at most a logarithmic factor, are summarized in Tab.~\ref{tab:lu}, where there are $\Theta(m^r)$ tiles of one color, say black, with $0 \leq r \leq d$ ($d$ is the grid dimension), and $p \geq 1$ denotes the number of escorts. w.h.p. means \emph{with high probability}.

\textbf{Note}: our makespan lower bound only applies to the case in which the goal configuration has the tiles sorted e.g. in row-major ordering, which is the important case for many applications, such as in the efficient design of automated garages.
Our upper bound applies to all initial/goal configurations by using the sorted ordering as an intermediate goal from the initial/goal configurations, so we will restrict our attention to the sorted case below.

\begin{table}[h!]
    \caption{Makespan lower and upper bounds for \bgsp. This is the same as for \pgsp and \cgsp where the black tiles correspond to all tiles but the largest group of colored tiles.}
    \centering
    \begin{tabular}{c|c|c}
    Black Tiles \, & Lower bound (w.h.p.) & Upper bound  \\
     \hline
    $0 < r < d$ & \,$\Omega\left(m + \dfrac{m^{1 + r[\frac{d-2}{d-1}]}}{p} + \dfrac{m^r}{p}\right)$\, & \,$O\left(m \log m  + \dfrac{m^{1 + r[\frac{d-2}{d-1}]}}{p}+ \dfrac{m^r \log m}{p}\right)$ \\
    $r = 0$ & $\Omega(m)$ & $O(m)$ \\
    $r = d$ & $\Omega(m^d / p)$ & $O(m^d / p)$
    \end{tabular}
    \label{tab:lu}
\end{table}
Our study of these further \gstp generalizations is motivated by their relevance to many high-utility applications in existing and next-generation autonomous systems. As one example, in autonomous warehouses where many mobile robots ferry goods around \cite{wurman2008coordinating}, robots that do not carry anything are anonymous (i.e., having the same color) and robots carrying packages are each unique, mirroring the settings in \pgsp. As another, in a future autonomous garage \cite{guo2023toward} at an airport where cars are parked and moved around by mobile robots going under them, during off hours, we may want to sort vehicles based on the date they are expected to be picked up to minimize the time needed to retrieve these vehicles during peak activity time, leading to a situation where \bgsp and \cgsp can serve as accurate abstract models.

\section{Preliminaries}
\subsection{Colored Generalized Sliding-Tile Puzzles}
In the \emph{generalized sliding-tile puzzle} (\gstp) \cite{gozon2024computing}, $n < m_1m_2$ tiles, uniquely labeled $1, \ldots, n$, lie on a rectangular $m_1 \times m_2$ grid $G=(V, E)$. 
A \emph{configuration} of the tiles is an injective mapping from $\{1, \ldots, n\} \to V = \{(v_x, v_y)\}$ where $1 \le v_x \le m_2$ and $1 \le v_y \le m_1$. 
Tiles must be reconfigured from a random configuration $\mathcal S=\{s_1, \ldots, s_{n}\}$ to some ordered goal configuration $\mathcal G=\{g_1, \ldots, g_{n}\}$, e.g., a row-major ordering of the tiles, subject to certain constraints. 
Specifically, let the \emph{path} of tile $i$, $1\le i \le n$, be $p_i : \mathbb{N}_0 \to V$.
Then \gstp seeks a \emph{feasible path set} $P = \{p_1, \ldots, p_n\}$ such that the following constraints are met for all $1\le i, j\le n$, $i \ne j$ and for all time steps $t \geq 0$: 
\begin{itemize}
    \item[\textbullet] Continuous uniform motion: $p_i(t+1) = p_i(t)$ or $(p_i(t+1), p_i(t)) \in E$,
    \item[\textbullet] Completion: $p_i(0) = s_i$ and $p_i(T) = g_i$ for some $T\ge 0$,
    \item[\textbullet] No meet collision: $p_i(t) \neq p_j(t)$,
    \item[\textbullet] No head-on collision: $(p_i(t)=p_j(t+1) \land p_i(t+1) = p_j(t)) = false$, and
    \item[\textbullet] Corner-following constraint: let $e_i(t) = p_i(t+1) - p_i(t)$ be the movement direction vector. If $p_i(t+1) = p_j(t)$, then $e_i(t) \not\perp e_j(t)$.
\end{itemize}

In the \emph{colored generalized sliding-tile puzzle} or \cgsp, tiles take on one of $k \ge 2$ \emph{colors}. We call $k = 2$ case the \emph{binary generalized sliding-tile puzzle} or \bgsp.
For the goal configuration of \cgsp, we do not fix it but require the tiles be arranged such that each color is a contiguous block in some layered manner, e.g., if there are $15$ tiles on a $4 \times 4$ board, $4$ white and $11$ black, then the goal could have white tiles occupy the first row and the black tiles located at the lower three rows, with the escort occupying the lower right spot. Or the white tiles could occupy the leftmost column. 
We call the case where there are unique tiles for each of the colors $1, \ldots, k - 1$ and $n - k + 1$ tiles with color $k$ the  \emph{partially-colored generalized sliding-tile puzzle} or \pgsp.

Let $T_P$ be the smallest $T\ge 0$ such that the completion constraint is met for a given path set $P$. Naturally, it is desirable to compute $P$ with minimum $T_P$. 
We define the decision version of makespan-optimal \cgsp/\bgsp/\pgsp  as: 

\vspace{2mm}
\noindent\mocgsp/\mobgsp/\mopgsp\\
\noindent INSTANCE: A \cgsp/\bgsp/\pgsp instance and a positive integer $K$.\\
\noindent QUESTION: Is there a feasible path set $P$ with $T_P \le K$?

\begin{remark}
As a \gstp variant, \cgsp allows many tiles to move simultaneously in a time step. Visually (see, e.g., Fig.~\ref{fig:escort}), this translates to the simultaneous ``teleportation'' of the escorts along non-intersecting straight lines. Such teleportation forms the basic moves in \cgsp. As this is frequently used throughout the paper, it is beneficial to keep it in mind when reading the paper.     
\end{remark}

\begin{figure}
    \centering
    \includegraphics[width=0.9\columnwidth]{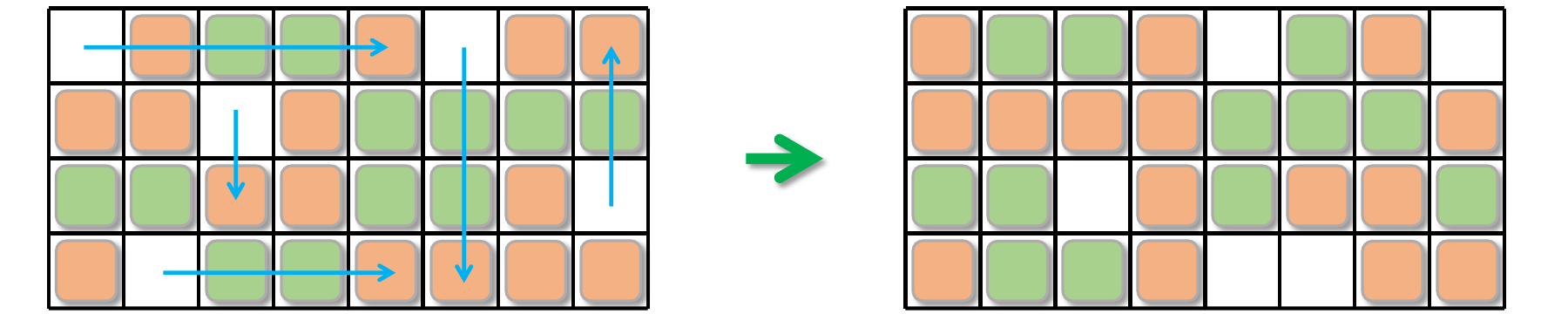}
    %{figs/MOGSTP Outline assembled.png}
    %\put(11.3, 2){{\small (a)}}
    %\put(67.8, 2){{\small (b)}}
    \caption{Illustration of escort teleportation in a \bgsp instance. The arrows show the teleportation intents, which are executed in a single step by moving tiles synchronously in opposite directions.}
    \label{fig:escort}
\end{figure}

\begin{remark}
Due to limited space, this paper mainly focuses on the case of $m = m_1 = m_2$; the presented results have straightforward generalizations to $m_1 \ne m_2$, e.g., via adding padding one dimension. In our experience, the case of $m = m_1 = m_2$ can often be the one yielding the worst bounds \cite{guo2022sub15}.
\end{remark}

\begin{remark}
Most of the results for 2D grids readily generalize to high-dimensional grids, e.g., a $d$ dimensional $m\times \ldots \times m$ grid. When applicable, the generalization will be briefly discussed. With a slight abuse of notation, we refer to the high-dimensional problems using the same names as the 2D setting with prefixes, e.g., $d$-dimensional \cgsp and \mocgsp. In higher dimensions, cubes, instead of tiles, are being slid. 
\end{remark}

\subsection{The Sand Castle Game}
In developing the hardness for single-escort \mocgsp, we make use of the \emph{sand castle game} (\scg) to help with the explanation. In \scg, we are given an $m_1 \times m_2$ grid with an initial and final configuration consisting of sand (black tiles) and holes (white tiles).
In transforming the first configuration to the second, the only allowed move in \scg is the repeated execution of \emph{collapse turns} or simply \emph{turns}. In a turn,  we pick an \emph{open tile}, which is a tile (black or white) that has a white tile as the rightmost element in the same row, and collapse it (sand filling holes) by shifting all the tiles to the right including the open tile by one tile to the right. Then all the tiles directly above the chosen open tile are moved down by a tile.
The above sequence of operations is referred to as \emph{collapsing a (open) tile}. 
The game ends when we reach the final configuration (a win) or when there are no more open tiles (a loss).

\section{Hardness of \mocgsp/\mobgsp/\mopgsp}
The hardness of makespan-optimal \gstp \cite{gozon2024computing} implies the hardness of \mocgsp in the special case $k = n$, though with a varying number of escorts.
Adapting the proof can also lead to a hardness proof of \mopgsp for the more general problem.
However, these results require a large number of escorts and do not extend to \mobgsp.
Here, we establish the fine-grained hardness of \mobgsp, \mopgsp, and \mocgsp, even in the most constrained setting with a single escort. Extensions further allow nearly arbitrary composition ratios for the colored tiles, any number of escorts, and high-dimensional settings.

\subsection{Constructing Single-Escort \mobgsp from 3SAT}\label{sec:reduction}
Given a 3SAT \cite{karp2010reducibility} instance with variables $x_1, \ldots, x_N$ and clauses $C_1, \ldots, C_M$, we will construct a corresponding single-escort \mobgsp instance. The instance will embed an \scg with ``passageways'' on the top and right
for the escort to simulate a turn (collapsing an open tile) of \scg every four moves abiding by the CFC.
We define the following variables:
\begin{align*}
    \ell = M + 2N + 2, \quad
    h = 2N\ell + 4N + 2, \quad
    q = \ell[4N(h+1) + 15].
\end{align*}
Roughly speaking, $\ell$ will be the length of our variable gadget, $h$ will be the distance needed for a clause tile to fall into the \emph{satisfaction position}, and $d$ will be the length of the \emph{reward} to enforce clause satisfaction.
Let $r = 4N(h+1) + h + q + 16$ and $c = q + M + 2N + 2$ be the number of rows and columns of the \scg.
Let $m = r + 2rc + 1$ and construct an $m \times m$ grid $B$ for our single-escort $\mobgsp$ instance (see Fig.~\ref{fig:bsnpuz}). The bottom left $(r + 2rc) \times c$ subgrid is denoted as $G$ and within it denote $H$ as the bottom $r \times c$ subgrid in which the \scg will be played. The tiles outside $H$ are needed to enforce the rules.
Above $G$ is a row of black tiles with a tail of length $w$, the number of white tiles in the SCG instance/moves needed to win (to become precise later) to ensure that the goal configuration is also the sorted configuration.
An escort is placed in the top right corner, with white tiles placed everywhere else.

\begin{figure}[h]
    \centering
    \includegraphics[width=\textwidth]{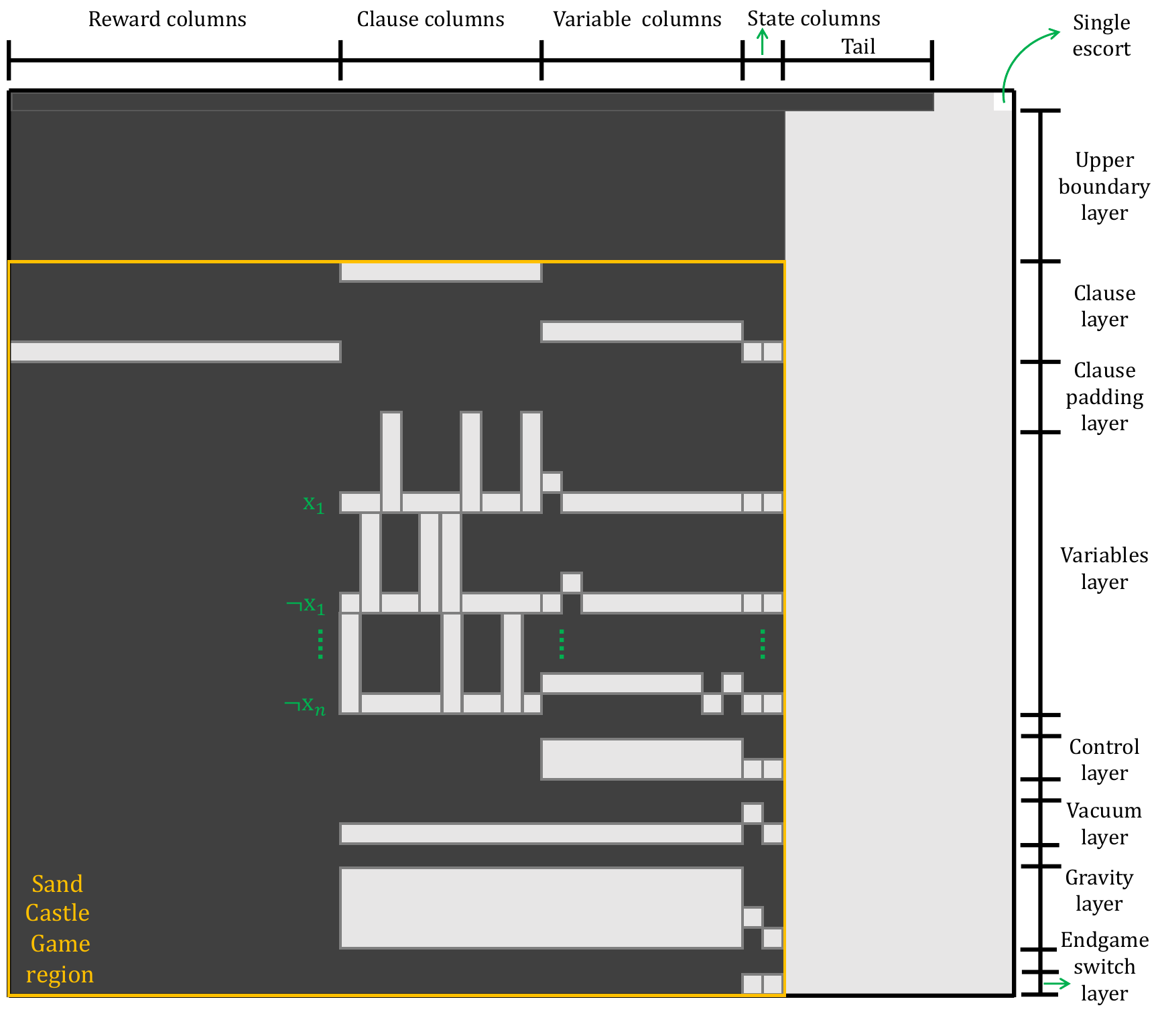}
    \caption{Illustration of the \mobgsp instance constructed from a 3SAT instance, showing the initial configuration. $G$ corresponds to the smallest rectangular region enclosing all the black tiles aside from the top row. $H$, the \scg board, is highlighted in orange. Not drawn to scale.}
    \label{fig:bsnpuz}
\end{figure}

We assign black and white tiles to the grid $G$ to describe the initial configuration of the \mobgsp instance.
$G$ has, from top down, the \emph{upper boundary layer}, \emph{clause layer}, \emph{clause padding layer}, \emph{variables layer}, \emph{control layer}, \emph{vacuum layer}, \emph{gravity layer}, and \emph{endgame switch layer}.
The upper boundary layer and the clause padding layer are entirely black tiles. 
Padding layers consisting of a single row of black tiles are added between each pair of consecutive layers neither of which are all black tiles to prevent interference.
We label the two rightmost columns as the \emph{state columns}, the next $2N$ as the \emph{variable columns}, the next $M$ as the \emph{clause columns}, and the last $d$ as the \emph{reward columns}.

The \emph{upper boundary layer} consists of a block of $2rc \times c$ black tiles. Its role is to prevent white tiles from bubbling up from the covered $c$ columns as the single escort moves to enforce the rules of the \scg.

The \emph{clause layer} consists of $h+1$ rows. The state columns have white tiles at the bottom (row) while the variable columns have a row of white tiles one row above. The clause columns have a row of white tiles at the top and the reward columns have a row of white tiles at the bottom. All other tiles are black.
The variable columns enforce a boolean choice of the variables, and the white tiles in the clause columns will simulate clause satisfaction by moving to the bottom row to allow for the leftmost $q$ \emph{reward tiles} to be shifted out.

The \emph{clause padding layer} consists of $q$ rows of black tiles.

The \emph{variables layer} consists of $2N(h+1)$ rows, $h+1$ for each literal $x_i$ and $\lnot x_i$; we call these individual blocks of rows \emph{literal gadgets}.
For each literal gadget, the state columns have white tiles at the bottom.
For the variable columns for a literal $x_i$, from the left, the first $2(i-1)$ columns have white tiles along the second row, and the last $2(n-i)$ columns have white tiles along the first row.
In addition, the ($2i-1$)th column has a white tile on the second lowest row while the ($2i$)th column has a white tile on the bottom row; $\lnot x_i$ differs in that columns $2i-1$ and $2i$ are swapped.
In the clause columns, the bottom row is filled with white tiles, and for every clause $C_j$ in which the literal participates, assign white tiles to the entire $j$th column from the left; we call these \emph{clause ropes}.
All other tiles are black.
The variable columns are designed to ensure that only one of $x_i$ or $\lnot x_i$ is true, i.e. it has a line of white tiles that can reach the clause columns; these allow for the white tile corresponding to the clauses in which the literal participates to fall onto the same level as the leftmost $d$ reward tiles in the clause layer to simulate clause satisfaction.

The \emph{control layer} consists of $2$ rows. 
The state columns have white tiles on the bottom, while the variable columns are filled with white tiles.
All other tiles are black. This layer allows for the boolean choice of the variables.

The \emph{vacuum layer} consists of 2 rows.
In the state columns, the right column has a white tile on the bottom row whereas the left column has a white tile on the row above.
The variable and clause columns contain a row of white tiles at the bottom, and all other tiles are black.
This layer will be used in conjunction with the gravity layer below to shift out all white tiles assuming that all the clauses are satisfied in the boolean formula.

The \emph{gravity layer} consists of $2N(h+1)+6$ rows.
As in the vacuum layer, the state columns have the right column containing a single white tile on the bottom row and the left column with a white tile on the row above. The variable and clause columns are filled with white tiles, and all other tiles are black.
This layer will help bring down all white tiles below the clause layer to be collapsed by the vacuum layer.

The endgame switch layer consists of a single row with two white tiles on the two rightmost columns with all other tiles black.
This will ensure that all clauses are satisfied before the white tiles below the clause layer are shifted out.

The above description, clearly computable in polynomial time, completely describes $G$, since accounting for all layers and padding, we have described $2rc + [h+1] + q + [2N(h+1)] + 1 + [2] + 1 + [2] + 1 + [2N(h+1) + 6] + 1 + [1] = 2rc + 4N(h+1) + h + q + 16 = 2rc + r$ rows, which completes the description of the \emph{initial configuration}.
The \emph{final configuration} of the \mobgsp instance is the sorted configuration with the black tiles on the left (columns of $G$ and $H$) and white tiles on the right, with the escort in the top right corner.

The final piece of information needed to obtain a \mobgsp instance is the upper bound $K$: if $w$ is the number of white tiles in $G$, then let $K = 4w$.
This is because the escort will simulate a turn of the \scg in four jumps.

\subsection{Properties of the Constructed \mobgsp Instance}\label{sec:equiv}
With the \scg and \mobgsp instance constructed, consider any solution satisfying the \mobgsp instance (the solution is a witness of the \mobgsp being satisfiable; with a slight abuse of notation, we simply refer to this as a \emph{solution} to the \mobgsp instance).
We note that the thick upper boundary layer full of black tiles forces white tiles to go around it during the reconfiguration.
Because there is a single escort with a restrictive makespan, white tiles must be cycled around in the counterclockwise direction one at a time to be pushed out of the Sand Castle.
Equivalently, the escort will be ``rotated'' in a clockwise direction, with each rotation along the corners of a rectangle on the $m \times m$ board. This forces the solutions to \mobgsp to simulate turns of the \scg. We have (see the Appendix for a complete proof):
\begin{lemma}\label{l:mobgsp-scg}
A solution to the constructed \textup{\mobgsp} instance corresponds to a winning set of turns of the embedded \scg instance. 
\end{lemma}

The key utility of the \scg is that it nicely abstracts the allowed moves in satisfying the \mobgsp to these tile collapsing turns, transforming the reduction to showing the correspondence between satisfying assignments for 3SAT and wins for \scg.
In fact, our results will show that the \scg is NP-hard as a decision problem and of independent interest.

In the forward direction, we prove: 

\begin{lemma}\label{l:sat-mobgsp}
A satisfying set of variable assignments to the \textup{3SAT} instance induces a solution to the reduced \textup{\mobgsp} instance.
\end{lemma}

The complete proof is given in the Appendix. As a sketch, from a satisfying 3SAT assignment, for each variable $x_i$, using \scg terminology, we can work on layers above the vacuum layer to collapse some white tiles in the control layer, the variables layer, and the clause layer, so that all white tiles in the clause layer will align into a single row. After collapsing all white tiles in the clause layer, cleanup can be performed to shift out all the rest to yield a win for the \scg and therefore, a solution to the \mobgsp.

In the reverse direction, we have: 

\begin{lemma}\label{l:mobgsp-sat}
A solution to the constructed \textup{\mobgsp} instance induces a satisfying set of variable assignments to the \textup{3SAT} instance.
\end{lemma}

Again, the complete proof is given in the Appendix. As a sketch, this direction is more involved with the gist being that
by construction, a solution must collapse all the white tiles in the clause layer via. the accessible variable and control layers below before switching to the endgame mode to cleanup the rest of the white tiles.

\subsection{Concluding Hardness}
The polynomial-time construction of the \mobgsp instance from a 3SAT instance, combined with Lemma~\ref{l:mobgsp-scg}-Lemma~\ref{l:mobgsp-sat}, immediately implies that \mobgsp is NP-hard. Note that a feasible solution to the \bgsp instance in the \mobgsp has a makespan of $O(m^3)$. To see this, if we do not need to simulate the \scg, we can use the escort to ``bubble up'' each white tile in $G$ through all the layers. In a bubbling-up operation, the escort is teleported above the target white tile and swapped with the white tile (pulling it up), then moved around the same white tile to be above the white tile again and repeat. Bubbling up each white tile takes $O(m)$ steps and there are at most $m^2$ white tiles.
Alternatively, we could just use the algorithm for \gstp.
Therefore, \mobgsp is in the complexity class NP and we have:

\begin{theorem}\label{t:semobgsp}Single-escort \textup{\mobgsp} is NP-complete.
\end{theorem}

Since $\mocgsp$ contains $\mobgsp$ instances with the \gstp algorithm still applicable, we get

\begin{corollary}
Single-escort \textup{\mocgsp} is NP-complete.
\end{corollary}

The construction can also be used to show that \mopgsp \emph{with arbitrary goal configurations} is NP-complete rather than the sorted goal configuration variant, since we do leverage the permutability of the black tiles in the top row.

\begin{corollary}
    Single-escort \textup{\mopgsp} with arbitrary goal configurations is NP-complete.
    The problem when restricting to goal configurations of corners of the grid (as rectangles) still remains NP-complete.
\end{corollary}
\begin{proof}
    Again, $\mopgsp$ is in NP since we can apply the \gstp algorithm.
    We use nearly the same reduction as for \mobgsp, but modify it slightly.
    By replacing the top row with the tail with white tiles and by setting the goal configuration to be the one in which all the black tiles move downwards in their respective column, one can verify that the hardness for \mobgsp with arbitrary goal configurations is still NP-hard, and it can be easily adapted to a \mopgsp instance by labeling the black tiles, leaving the white tiles as the unlabeled tiles, with the goal instance having all the labeled tiles in the initial instance ``falling down'' following their initial orders.

    To show the hardness when moving labeled tiles to a corner, one can add white tiles above the strip of white tiles in the reward columns so that all the columns of the \scg instance have the same number of white tiles, adjusting the dimensions of the instance slightly without affecting the \scg simulation.
    Then the goal configuration will be a rectangle of labeled tiles in the bottom left corner.
    Furthermore, by simply padding with extra labeled and unlabeled tiles, one can further restrict the dimensions of the corner to have dimensions with fixed powers of $m$, even nearly sorted corner configurations, e.g. $(m - m^\epsilon) \times (1/2)m$ for $0 < \epsilon < 1$ a constant.

~\qed
\end{proof}

In single-escort \cgsp, let the fraction of tiles for a color $i \in \{1, \ldots, k\}$ among all tiles be $f_i$, $\sum_{1\le i \le k}f_i = 1$. 
For example, in a \bgsp, it can be that black tiles and white tiles each have $f_i \approx 0.5$.
We can even take the $f_i$ to have fixed powers of $m$.
We can extend our proof strategy to show that for a fixed color $k$, single-escort \mocgsp remains NP-complete with the distribution of $\{f_1, \ldots, f_k\}$ arbitrarily close to any fixed distribution. 
Essentially, we can first pick two proper colors $i, j$ as black and white, and construct a \mobgsp instance from a 3SAT instance done in Sec.~\ref{sec:reduction}. Then, we can expand the \mobgsp instance on the top and the right and assign the added tiles proper colors (possibly more black and white tiles as well) to approximate the desired colored tile ratios. The goal configurations of these added tiles are the same as their initial configurations. 
For example, we may expand the $m\times m$ grid to an $m^{\gamma} \times m^{\gamma}$ grid for arbitrary but fixed ${\gamma}$; the construction time and instance checking time remain polynomial with respect to the original 3SAT instance size. 
We note that minor cautions are needed to maintain the \scg mechanics. We have shown: 

\begin{theorem}
Single-escort \textup{\mocgsp} remains NP-complete if tiles of color $i \in \{1, \ldots, k\}, k \ge 2$ occupy an approximately $f_i$ fraction of the puzzle game board. 
\end{theorem}

By adding additional independent escort jobs, hardness results hold for an arbitrary number $p$ of escorts, as long as $p$ is polynomial in the size of the problem.

\begin{theorem}\label{t:pescorts}
For $p \ge 2$ and polynomial in problem size, $p$-escort \textup{\mobgsp} and thus $p$-escort \textup{\mocgsp} remain NP-complete. 
\end{theorem}
\begin{proof}
Construct an instance as in Fig. \ref{fig:mobgsp-multi}, placing an SCG instance rotated 90 degrees counterclockwise in the bottom right corner of the grid. 
Fill the area left of it with black tiles except with some small constructed independent escort jobs, making all other tiles white except for the escorts in the top left corner.
Using the same variables as in the \mobgsp reduction, the \scg instance now occupies a $c \times r$ subgrid, with independent escort jobs placed at a distance of $8w$ from each other and the \scg to ensure that they remain independent from one another, with $2wp$ black tiles needed to be absorbed into the black tiles and $2w(p-1)$ white tiles needed to be moved out from it, altogether occupying $8w(4wp - 2w)$ columns.
Additionally, add $p$ extra columns for the starting positions of the escorts, aligned on the main diagonal from the top left corner to ensure that they do not conflict with one another.
Thus, there will be $p+1$ rows of white tiles along the top.
The final configuration will be the sorted configuration with the black tiles on the bottom and escorts placed above some subset of the protruding black tiles, with the same total makespan $K = 4w$.

Now we analyze based on the number of escort ``corners,'' which are the positions at which the escorts stop throughout the movement.
In this instance, each of the escort jobs, as well as moving a white tile out of the \scg instance, will have some corresponding escort jump in the same column that either moves the white tile out or the black tile in, using two escort corners.
Furthermore, all of these corners among the escort jobs and between those and the ones for the \scg instance will be disjoint, since they are placed at a distance of $8w$ from each other.
Since the total number of escort corners is $p(4w)$, $4wp - 2w$ are already used on the independent escort jobs, leaving $2w$ left for the \scg instance, each shifting out a white tile requiring two corners.
Thus, the \scg instance permits $w$ escort moves, each of which results in the escort jumping to the left below a protruding black tile, effectively simulating a turn of the \scg instance.
Thus, a successful tile routing will simulate the \scg instance, and a solution to the \scg instance readily results in a successful tile routing by setting $p-1$ of the escorts to fulfill the independent escort jobs and the other escort to simulate the \scg instance and resolve its share of $w$ protruding black tiles.
Minor modifications may be needed to ensure that movement goes as planned, e.g. if $p$ is sufficiently large, then we would need to widen the sea of black tiles on the bottom so that escort jumps can be made to not conflict with one another.
~\qed
\end{proof}

\begin{figure}[h]
    \centering
    \includegraphics[width=\textwidth]{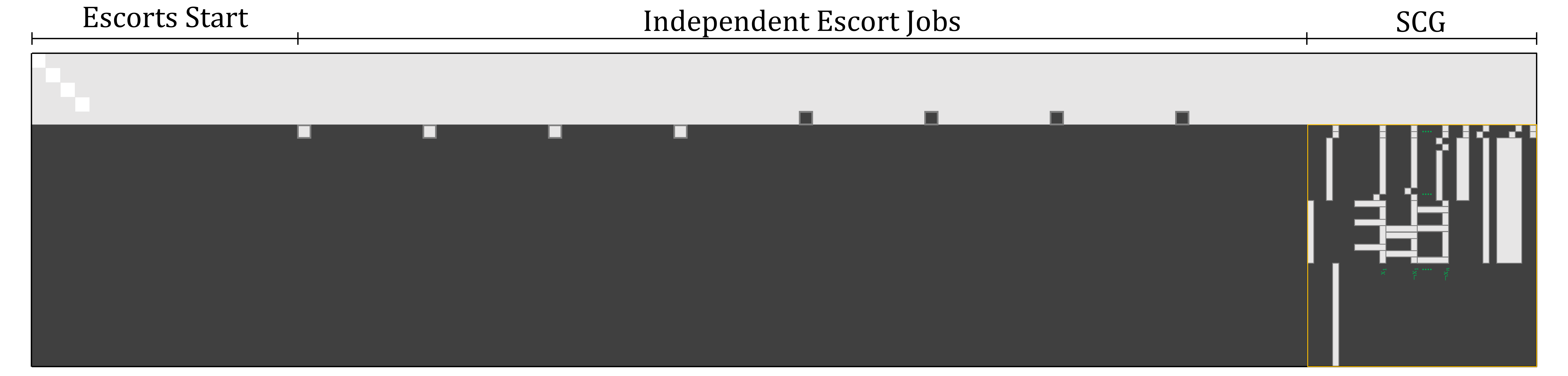}
    \caption{Illustration of handling multi-escort \mobgsp instances by adding additional independent escort jobs.
    Only the lower part of the overall grid is shown, with the upper part truncated.
    Not drawn to scale.}
    \label{fig:mobgsp-multi}
\end{figure}

\begin{corollary}\label{c:mopgsp-multi}
For $p \ge 2$ and polynomial in problem size, $p$-escort \textup{\mopgsp} with arbitrary goal configurations is NP-complete.
Again, the problem when restricting to goal configurations of corners of the grid (as rectangles) still remains NP-complete.
\end{corollary}
\begin{proof}
Again, to deal with the issue of the labeled tiles being different, reflect the \scg instance to face the right, pad it with extra black tiles, and introduce a column of white tiles on the right.
Since we intend the escort simulating the \scg instance to use that side route instead, we only need $2w(p-1)$ of the protruding black tiles instead.
To ensure that the escort movements left of the \scg don't interfere with one another, we can pair the independent escort jobs from the middle and connect the pairs via nonintersecting grid-aligned paths, from which we can then label the black tiles from their starting configuration and goal configuration such that they move once if they are on a path and remain stationary if not.
Then the starting and goal configurations in the \scg area following the ``falling down,'' though now to the left, procedure, with the escorts ending up above the protruding black tiles as well as the top right corner.

Again, we have at most $4w$ escort corners reserved for the \scg instance, and pulling out a white tile comes with two escort corners + two more to move an escort above the boundary between the black and white tiles, so again we will have a faithful simulation of the \scg instance.
A solution to the \scg instance will readily have a corresponding solution to the \mopgsp instance.
The one can make minor modifications again to obtain the corner goal configuration requirement as needed.

~\qed
\end{proof}

The hardness results can be extended to high-dimensional grids by simply simulating the problem in one of the 2D slices, padding the additional dimensions with black tiles:

\begin{theorem}
For $p \ge 1$ and polynomial in the problem size, $p$-escort \textup{\mobgsp} and $p$-escort \textup{\mocgsp} remain NP-complete on $d$-dimensional grids for fixed integer dimensions $d \ge 2$.
The NP-completeness also holds for $p$-escort \textup{\mopgsp} with arbitrary (or corner) goal configurations.
\end{theorem}

\section{Makespan Bounds}
We proceed to provide makespan lower and upper bounds, as stated in Tab.~\ref{tab:lu}, that differ by at most a logarithmic factor for each of \bgsp, \pgsp, and \cgsp, starting with the single escort case in two dimensions where $m = m_1 = m_2$.
We also extend and solve the higher dimensional versions.
Since the primary difficulty is the underlying \bgsp problem in each formulation, we center on this case.
All upper bounds provided here come with a low polynomial time algorithm that computes plans corresponding to these upper bounds. In other words, if desired, fast algorithms can be readily implemented to compute tile movement plans with the stated upper bounds. With that said, we do not further discuss the issue of computation time in this work. 

For \pgsp/\cgsp, the lower bounds of \bgsp directly apply since \pgsp/\cgsp must use at least as many moves. The same upper bounds also hold but the reason is slightly more subtle: the bounds are obtained by running the \bgsp algorithm and then utilizing the \gstp algorithm \cite{gozon2024computing} on the subgrid with distinguishable (resp. non-majority colored) tiles, utilizing each of the additional escorts to maximal capacity as in \cite{gozon2024computing} but extended to higher dimensions through the higher dimensional Rubik Table \cite{szegedy2023rubik}.
One can also obtain more precise constants as in \gstp \cite{gozon2024computing}. We do not expand on these here given the limited space and the limited importance.

We note that the cases of $r= 0$ and $r = d$

are fairly straightforward to analyze:
for \bgsp/\pgsp/\cgsp, when $r = 0$, we have a constant number of black (non-majority colored) tiles to move, so we can transport each individually (e.g., through bubbling them around) to obtain a constant factor approximation algorithm. The case of $r = d$ has a lower bound coming from the sum of Manhattan distances that can be constant factor matched by solving the whole as a generalized \gstp instance \cite{gozon2024computing}.

With the case of $r=0$ and $r=d$ cleared, in what follows, we only work with the $0 < r < d$ setting. 

\subsection{Polynomial Time Upper Bounds}
For the upper bounds, we first outline the overall structure of the algorithm in the single escort case in two dimensions, but note that our approach is designed with the multi-escort setting in mind. 
As in the study of \gstp, it is more useful to view the routing problem through the teleportation or jumps made by the escort, in which the goal is to \emph{batch} as much of the necessary movement as possible.
To do this across the grid and to make our algorithm parallelizable, we utilize a divide-and-conquer approach.

For the main case $d = 2$, assume that $m = m_1 = m_2$, which is the case most important to dense robotics applications.
For simplicity, assume that $m$ is a power of two.
The algorithm then works in $\log_2 m$ iterations, where we initially consider the grid through each of the $1 \times 1$ cells.
Then in the $i$th iteration of the algorithm, we increase the granularity of the grid of $2^{i-1} \times 2^{i-1}$ squares to be a grid of $2^i \times 2^i$ squares, combining the four respective $2^{i-1} \times 2^{i-1}$ subsquares for each $2^i \times 2^i$ square, where we layer the black tiles from the bottom of the respective square rightwards and then upwards.
Note that the goal configuration can be arbitrary since we can run the algorithm again backwards to the sorted configuration.

The escort must travel along a continuous trajectory. Due to this, we need to add additional unit margins of white tiles to each of the squares to facilitate movement to ensure that our configuration is not ruined when the escort is moved across the grid.
However, this introduces problems for the early iterations of the algorithm, where certain $2^i \times 2^i$ squares may have more black tiles than their inner $(2^i - 2) \times (2^i - 2)$ square.
Thus, we additionally need to pre-process the instance by spreading out the black tiles across the grid, starting at a higher level of granularity with $8 \times 8$ squares instead.

Given the background, the underlying intuition of our overall approach is as follows:
by infusing enough regularity within our grid, we can utilize the escort to batch the movement across the grid, even among sparse instances.
Consider the process of merging four squares, and for simplicity, suppose we merge sequentially along the rows and then the columns.
Then when merging two squares of size $\ell \times \ell$ along the rows, the row underneath the bottom of the two squares is the same for pairs across the grid, so we can use it as a ``highway'' to transport black tiles from one square to another, only using a small number of time steps to move black tiles onto the highway and back out (which can be batched column-wise using the column alignment of the squares).
Similarly, when merging two rectangles of size $\ell_1 \times \ell_2$ along the columns, the column left of the two rectangles functions as the highway, where we need to additionally reorient the black tiles of a given square to be placed on the highway simultaneously using row moves before being moved downwards in parallel with column moves.

The analysis then boils down to \emph{assigning the cost of some time steps to certain black tiles} to have a low average cost, along with the base cost of running a divide and conquer algorithm.
For the case in which $m_1 \times m_2 = \Theta(m) \times \Theta(m)$ for some integer $m$, we must additionally handle border cases without much effort.
By the way the algorithm is set up, adding more escorts allows us to efficiently utilize each of the escorts and divide the work accordingly, whereas the case with higher dimensions is handled with a similar divide and conquer approach, instead now merging along each of the $d$ dimensions as needed.
This gives us algorithms that nearly match the high probability lower bounds provided in the next subsection.

\begin{theorem}[Single-Escort \bgsp Algorithm]\label{t:sebgspa}
A single-escort \bgsp instance of size $m_1 \times m_2 = \Theta(m) \times \Theta(m)$ with $B = \Theta(m^r)$ black tiles for $0 < r < 2$ can be solved in $O(m \log m + m^r \log m)$ time.
\end{theorem}

The complete proof can be found in the Appendix. The multiple escort case is largely similar. With multiple escorts, the difference is that parallelism can be readily introduced into the divide-and-conquer approach. 

\begin{theorem}[Multi-Escort \bgsp Algorithm]\label{t:mebgspa}
    A \bgsp instance of size $m_1 \times m_2 = \Theta(m) \times \Theta(m)$ with $B = \Theta(m^r)$ black tiles for $0 < r < 2$ with $p$ escorts can be solved in $O(m \log m + m^r \log m / p)$ time steps.
\end{theorem}

The complete proof can be found in the Appendix. These algorithms can be generalized to the case of arbitrary dimensions.

\begin{theorem}[Higher Dimensional \bgsp Algorithm]\label{t:hdbgspa} A $d$ dimensional \bgsp instance of size $m_1 \times \cdots \times m_d = \Theta(m) \times \cdots \times \Theta(m)$ with $B = \Theta(m^r)$ black tiles for $0 < r < d$ with $p$ escorts can be solved in $O(m \log m + m^{1 + r[\frac{d-2}{d-1}]} / p + m^r \log m / p)$ time steps.
\end{theorem}

Again, see the Appendix for the complete proof.

\subsection{Fundamental Makespan Lower Bounds}
For lower bounds, we directly work with a $d$-dimensional grid $m\times \ldots \times m$. The lower bound for \bgsp (Tab.~\ref{tab:lu}) is $\Omega(m + m^{1 + r[\frac{d-2}{d-1}]}/p+ m^r/p)$ for $p \ge 1$ escorts. It has three parts, each of which may dominate depending on $r$ and $p$. The first and last terms are easy to tally. 
For a random instance with $\Theta(m^r), r \in (0, d)$ black tiles, with high probability, a black tile must take at least $\Omega(m)$ time steps to ``swap out'' a white tile, regardless of the number of available escorts. This gives the first $\Omega(m)$ term. 

Similarly, with high probability, the total Manhattan distance that must be traveled by $m^r$ black tiles in a random \bgsp instance is $\Omega(m^r \cdot m)$. 
To get this, suppose black tiles must relocate to the bottom of the grid (i.e., $m^r/m^d < 1/4$, otherwise, flip the colors), we lower bound the sum of Manhattan distances by the distance of a black tile to the bottom fourth of the grid.
If there are $B$ black tiles, let $X_i$ be the distance to the bottom of the grid, which is 0 if it is in the bottom fourth and expected to be $3m/8$ if it is in the upper three-fourths of the grid.
Then the sum of Manhattan distances is lower bounded by $X = X_1 + \cdots + X_n$, which has $\E[X] = 3mB / 8$ and $\Var(X) = B\Var(X_i) + B(B-1) \Cov(X_i, X_j) \leq B \Var(X_i)$, as the $X_i$ and $X_j$ are negatively correlated.
We can compute $\Var(X_i) \leq \E[X_i^2] = (1^2 + 2^2 + \cdots + (3m/4)^2) / m = O(m^2)$.
Then Chebyshev's inequality gives that $\Pr[|X - \E[X]| \geq \E[X] / 2] \leq 4\Var[X] / \E[X]^2 \leq O(1 / B^2) = O(1 / m^{2r})$, i.e. the sum of Manhattan distances is $\Omega(m B) = \Omega(m \cdot m^r)$ as required.
In each time step, this distance can be decreased by at most $pm$, yielding the last lower bound term of $\Omega(m^r/p)$.

We are left to show the third part of the lower bound, $\Omega(m^{1 + r[\frac{d-2}{d-1}]}/p)$, which is less straightforward and also more interesting. 
Consider the relaxation where the $p$ escorts are replaced with white tiles, except now, to simulate the escort movement, we are allowed to pick any generalized column and move the black tiles in whichever direction we want, even permitting them to move to the same square, performing this operation once for each escort in sequence.
Instead of trying to sort the grid, we seek to \emph{collapse} (here, ``collapse'' bears a meaning different from the collapsing operations in \scg) all the black tiles into a single square.
Note that this is a valid relaxation of \bgsp up to an additive factor of $O(B/p)$, since we can indeed simulate the escort movement through the allowed operations, then subsequently collapsing all the $B$ black tiles trivially in $O(B)$ steps, divided by the number of escorts $p$.
Thus, the grid collapse game can be solved in at most the minimum amount of time to solve the \bgsp instance plus $O(B/p)$ steps (even though it's most likely the case that the grid collapse game can be solved faster), from which demonstrating the $\Omega(m^{1 + r[\frac{d-2}{d-1}]} / p)$ gives the lower bound for \bgsp as $m^{1 + r[\frac{d-2}{d-1}]} \gg m^r$ when $0 < r < d-1$, which is the case we care about (when $r \geq d-1$, the sum of Manhattan distances in the last term of the lower bound provides a higher bound).

The hardness of the \bgsp instance arises from the underlying lattice-like structure of black tiles that must be collapsed into a single black tile.
We use this bottlenecking difficulty apparent in the upper bound divide and conquer solution as our approach. We have (see Appendix for the proof):

\begin{theorem}\label{t:lb3}
A random $d$-dimensional $m \times \cdots \times m$ \bgsp instance with $B = \Theta(m^r)$ black tiles and $p$ escorts for $0 < r < d-1$ takes $\Omega(m^{1 + r[{\frac{d-2}{d-1}]}} / p)$ steps.
\end{theorem}

\section{Conclusion and Discussions}
In this work, we have investigated the colored generalized sliding-tile puzzle (\cgsp) and its prominent variants, \bgsp and \pgsp. \bgsp, \pgsp, and \cgsp formulations model after and capture key features of many high-impact, large-scale robotics and automation-related applications. We prove that finding makespan optimal solutions for \bgsp, \pgsp, and \cgsp is NP-complete, even when there is a single escort. The hardness results are shown to hold when there is an arbitrary number of escorts, for arbitrary color compositions, and in higher-dimensional settings. We then analyze the makespan bounds for \bgsp, \pgsp, and \cgsp, establishing lower (fundamental limits) and upper (polynomial time algorithmic) bounds that differ by at most a logarithmic factor. 

Many open questions remain. Within the context of the current study, on the hardness side, we leave open the question of how to show that \gstp itself is NP-hard when there is a single escort (note that if we do not allow parallel moves, then this is shown in \cite{RATNER1990111}). The single escort setting should also extend to the setting with a constant number of escorts. 

On the bounds side, the makespan lower and upper bounds can differ by up to a logarithmic factor. We hypothesize that the lower bound can be further tightened as it is known that randomness assumptions can often introduce an additional logarithmic factor \cite{erdHos1961classical}. A more careful analysis leveraging additional tools may eliminate the logarithmic bound gap. 

Going beyond the current study, we didn't labor to pin down a concrete number for the constant factor, which we believe is unlikely to be practical. With that said, however, it remains interesting to push the envelope on fast algorithms that compute (near-)makespan-optimal solutions for \gstp and \cgsp variants, for which there are some existing efforts \cite{fekete2022computing,guo2023efficient}. We believe it would be of particular interest and practical relevance, to explore the impact of the number of escorts on practically achievable makespan upper bounds. The bounds from the current work help  guide such efforts. 

In our algorithmic study, the initial configuration is assumed to be random and the goal configuration have the colors ordered. The formulation is motivated by the sorting requirement of many practical applications. It is also interesting to consider arbitrary goal configurations and demonstrate tighter lower bounds as well as develop more specific algorithms beyond our sorting one.
As the setting mimics those studied in \cite{shor1991minimax}, it is likely to have drastically different complexity and makespan bounds.

Also of interest is examining tiles of different sizes such as those in Klotski, which appears more complex and likely requires new tools to be developed. 

\section*{Acknowledgement}\label{sec:ack}
We thank the reviewers for their insightful suggestions. This work is supported in part by the DIMACS REU program NSF CNS-2150186, NSF award IIS-1845888, CCF-1934924, IIS-2132972, CCF-2309866, and an Amzaon Research Award. 

\bibliographystyle{splncs04}
\bibliography{refs}
\newpage 
\pagebreak
\section*{Appendix}
\begin{proof}[Proof of \textup{Lemma~\ref{l:mobgsp-scg}}]
We show that the escort's movement corresponds to simulating moves in the \scg:
after every four moves of the underlying \bgsp, the board configuration will be transformed to another one corresponding to performing an allowed move in \scg, and the target configuration will be reached exactly when the \scg ends in a win.

Consider any \bgsp solution satisfying the \mobgsp instance.
Any initial white tile in $H$ can never pass through the upper boundary layer at any point of the tile routing. This is because the escort alternates between rows and columns when it jumps (otherwise it would be wasting a move), so at most $2w$ column jumps can be performed.
Each jump can either move the white tile upwards a tile or the upper boundary of $G$ downward a tile (mutually exclusive due to opposing directions of movement). 
Because $2w < 2rc$ (the upper boundary layer is $2rc$ thick), it is impossible to pop a white tile initially in $H$ above the upper boundary layer.
Thus, white tiles initially in $G$ must leave $G$ through the right side.

Now define $\ell_h$ to be the line separating the upper boundary layer of $G$ from the clause layer, and let $\ell_v$ be the right boundary of $G$; the escort must move in cycles of four jumps going clockwise around the induced quadrants formed by $\ell_h$ and $\ell_v$.
This is because to shift out one of the original white tiles in $G$, the escort must jump from the right side of $\ell_v$ to the left, and this can occur at most every $2$ escort row jumps since jumps alternate between rows and columns and a row jump must transport the escort back to the right of $\ell_v$.
Thus, every row jump in a solution must cross $\ell_v$ as $w$ white tiles must be shifted out with the available $2w$ row jumps, and since the escort must end in the upper right corner, right of $\ell_v$.
In addition, the jump directly after shifting out a white tile must have the escort jumping from below $\ell_h$ to above it, since if it didn't, the next row jump would leave $H$ below the line and increase the number of white tiles back to what it was before, making the solution invalid.

After every four moves in the \mobgsp, a turn in the \scg in $H$ will be simulated. The escort is only ever allowed to jump to $H$ through a row containing a white tile at the rightmost column of $H$, since if it didn't, then in $4w$ moves, it would be impossible to shift out $w$ white tiles.
Then after entering $H$ through a white tile, the escort must jump to above $H$, simulating the ensuing fall of the tiles above.
Furthermore, it must go the top row:
in order for the rightmost black tile in the tail to move left of $\ell_v$, the escort necessarily cycles between four stages (right, right different row, left different row, left) for all $w$ time steps.
All in all, a \scg turn will be simulated, and this is all the moves allowed.
In addition, the endgame conditions are nearly the same: a solution induces a winning set of turns on the \scg via the corner between the $(4i-2)$th and $(4i-1)$th escort jump.
On the other hand, a winning set of turns $t_1, \ldots, t_w$ to the \scg
can be transformed into an \mobgsp solution by having the escort be in the top right corner of $B$ after every four jumps and at the tile specified by $t_i$ after the $(4i-2)$th jump; this is indeed a solution since the final configuration of $G$ is obtained by removing the white tiles and letting the black tiles fall to fill the void.   
~\qed
\end{proof}

%###################################################

\begin{proof}[Proof of \textup{Lemma~\ref{l:sat-mobgsp}}]
By Sec.~\ref{sec:equiv}, it is enough to specify the open white tiles (connected to the right boundary through a contiguous row of white tiles) chosen in the \scg in $H$.
Let the 3SAT instance consist of variables $x_1, \ldots, x_N$ and clauses $C_1, \ldots, C_M$.
We will solve the \scg in three stages: the clause satisfaction stage, the reward stage, and the cleanup stage.

In the clause satisfaction stage, iterate from $i = 1, \ldots, N$.
If $x_i$ is true, then collapse the white tile in the bottom row of the $(2i-1)$th column (from the left) of the variables column in the control layer, and if it is false, do it on the $(2i)$th column;
then the corresponding literal gadget in the variable layer will have a consecutive row of white tiles through the variables column that can access the clause columns.
Then for each clause $C_j$ in which the literal participates, and in which no other previous literal in the algorithm also participates and is true, collapse the bottom white tile corresponding to column $C_j$ $h$ times.
At the end of the iteration, collapse the other white tile corresponding to variable $x_i$ in the control layer before moving to the next iteration.

As a result of the stage, the clause layer has the white tiles in the variable columns fall by one and the white tiles in the clause columns fall by $h$.
Thus, the white tiles in the clause layer will be in a single row.
Then in the reward stage, simply collapse the white tiles in the clause layer from left to right. 

In the cleanup stage, the white tiles in the two rightmost columns of the variables and control layer are collapsed first.
Afterward, the white tile in the left state column in the endgame switch layer is collapsed, which will connect the white tiles in each of the vacuum and gravity layers.
Then, in the clause and variable columns, iterate from the left column to the right and do the following: 
first, collapse the white tile in the vacuum layer.
Then while the gravity layer still has white tiles, collapse them, and collapse any white tiles that fall into the bottom row of the vacuum layer as a result of pulling the white tiles down through the gravity layer.
Note that this will collapse all the white tiles except those in the state column of the vacuum, gravity, and endgame switch layer since the gravity layer has $2N(h+1)+6$ white tiles in the clause and variables columns.
This means that, since all remaining white tiles are in the variable layer or below, the distance of any white tile in the clause and variable columns to the bottom row of the vacuum layer is at most $[2N(h+1)] + 1 + [2] + 1 + [2] -1 = 2N(h+1) + 5 < 2N(h+1) + 6$ resulting in it getting pulled down by the gravity layer and shifted to the right by the vacuum layer.
At last, collapse the two rightmost white tiles in the vacuum layer and then the gravity layer, collapsing the final white tile in the endgame switch layer.

The whole procedure shifts out all white tiles as a winning strategy to the \scg in $H$, inducing a solution to \mobgsp.
~\qed
\end{proof}

\begin{proof}[Proof of \textup{Lemma~\ref{l:mobgsp-sat}}]
Assuming a solution to \mobgsp and therefore a winning set of turns for \scg, at any point of \scg, the white tiles in the clause layer can never fall to be adjacent to white tiles in the variables layer and below, because the number of white tiles in the variables layer and below is upper bounded by the $[4N(h+1) + 15] \times [M + 2N + 2]$ bottom right subgrid of $G$ in which they are inside, which has $q$ tiles.
The padding layer between the clauses and variables layer requires a white tile above fall $\geq q+1$ tiles, which cannot happen.
Thus, the tiles in the clause layer must be shifted out via a white tile originally in the clauses layer.

To clear out the clauses layer, all white tiles in it must be on the same row, and collapsing must happen from the leftmost column to the right.
Suppose on the contrary that one of these conditions is not met, and consider the first instance.
Then before the collapse, all white tiles were on different columns, and so afterward, the column of the chosen open tile $c_c$ and to the right will have fewer white tiles than the number of columns, meaning that some column $c_e$ has no white tiles.
Denote $c_r$ as the number of white tiles to the right of $c_e$, and note that there is a white tile to the left of it since there was one left of $c_c$.
However, no white tile to the left of $c_e$ can be cleared out because, after $c_r$ operations using the tiles to the right, any white tile to the left can only travel $c_r$ columns to the right, preventing any of them from being shifted out since they cannot reach the right boundary of $H$ or fall to a layer below, forming a contradiction to the solution to the \scg.

Since all white tiles in the clause layer must belong to the same row to be shifted out, the row in which they must align is simply the bottommost row in the clauses layer. At most $q$ columns can be pulled down from the white tiles in the variables layer and below, but there are more than $q$ white tiles in the clauses layer.

To collapse white tiles in the vacuum and gravity layers, a white tile from the endgame switch layer below must bring down the white tile in the left state column since otherwise, the white tiles on those layers could not be shifted out due to the same argument as for the clauses layer.
Then, if the white tile in the endgame switch layer were to be used to bring down the left state column in a solution to \mobgsp, all of the white tiles in the clauses layer, except possibly one on the right boundary, must have been shifted out first.
This means that the variables and control layer and possibly the endgame switch layer must have been used to bring down the white tiles in the clauses layer to the same row first for the clauses layer to be cleared out;
however, no white tile from the endgame layer can be used since one is reserved for bringing down the second rightmost column, and bringing down any other column would misalign the bottom row of the gravity layer, making it impossible to clear out that layer.

Now in any solution to the \scg, it can \emph{activate} at most one of $x_i$ or $\lnot x_i$ for $i = 1, \ldots, n$.
By this, we mean that only one of the literal gadgets in the variables layer can have all its white tiles in the variables columns aligned with the white tiles in the state and clause columns.
This is because for this to happen, some subset of the variable columns must be each pulled down by one (via the control layer) as specified by which white tiles are raised in the corresponding variable gadget, and the subsets corresponding to $x_i$ and $\lnot x_i$ are not contained in each other.
In addition, a variable column cannot be pulled down more than one tile until the white tile in that column in the clause layer gets shifted to the right, which can only happen when the clause tiles have fallen to the bottom row in the clause layer.
Thus, a solution cannot also `cheat' (before bringing down the clause tiles) by lowering an entire literal gadget to any row below since that would require one variable column to shift down by at least two.

Now note that a clause tile can only fall into the satisfaction position via the clause rope of an activated variable gadget.
If not, then if a solution tried to pull it down using activated literal gadgets, it could pull it down by at most $\ell$ rows using the white tiles on the bottom row. If it tried using deactivated literal gadgets, it could pull it down by at most $2n + 2 < \ell$ white tiles due to the black tile in the way.
Thus, including the possibility of using the white tiles in the control layer, a clause tile can fall strictly less than $2N\ell + 4N + 2 = h$ in this manner, so it can never fall into the satisfaction position.
Note that a clause rope could then never fall to a literal gadget below since the clause tile would need to have fallen into satisfaction position first, so there is no way to `cheat' bringing down the clause tile.

Thus, in a \scg win, a literal must be activated and have the corresponding clause rope pull down the clause tile by at least a row.
Then each column $C_j$ has an activated literal $x_i$ or $\lnot x_i$ with a clause rope pulling it down, i.e. it belongs to that clause.
Thus, the set of activated literals induces a boolean variable assignment to $x_1, \ldots, x_N$ that satisfies each of the clauses $C_1, \ldots, C_M$.
In the case that neither $x_i$ nor $\lnot x_i$ is set activated, arbitrarily setting the variable $x_i$ to true does not change the satisfaction of the boolean formula.
Thus, a solution to the \mobgsp induces a solution to the \scg and thus the 3SAT.
~\qed
\end{proof}

\begin{proof}[Proof of Theorem~\ref{t:sebgspa}]
Assume for now that $m$ is a power of two; the general case is discussed at the end of this proof.
In the first stage, we seek to spread out the black tiles to have a small enough density in the entire grid. We utilize the following algorithm:
\begin{enumerate}
    \item Consider the black tiles from the top row to the 4th to the bottom row, from left to right, and suppose we are on tile $b$.
    If $b$ has two white tiles underneath it, move on to the next.

    \item If not, then use the escort to absorb the black tile(s) underneath into the row of black tiles without affecting the status of ``completed'' tiles.
    If the row of $b$ is full of black tiles, move on to the next row of black tiles.
\end{enumerate}
Note that this takes $O(B)$ steps as there are $B$ black tiles each one can only be absorbed into the line above once.
Then once this process is done, each partially completed row will have two white rows below it, and so at most $1/3$ of the rows can be partially completed.
Since $r < 2$, for large enough $n$, $B/m^2 < 1/9$, and so at most $1/10$ of the rows can be fully black.
Thus, at least $5/9$ of the rows is white, and so rerunning the algorithm in $O(B)$ steps along the columns and instead only considering a single tile underneath the current black tile ensures that only partial columns are constructed.
Thus, the number of black tiles in each square will have a density of at most 1/2.

\begin{figure}
    \centering
    \includegraphics[width=1\linewidth]{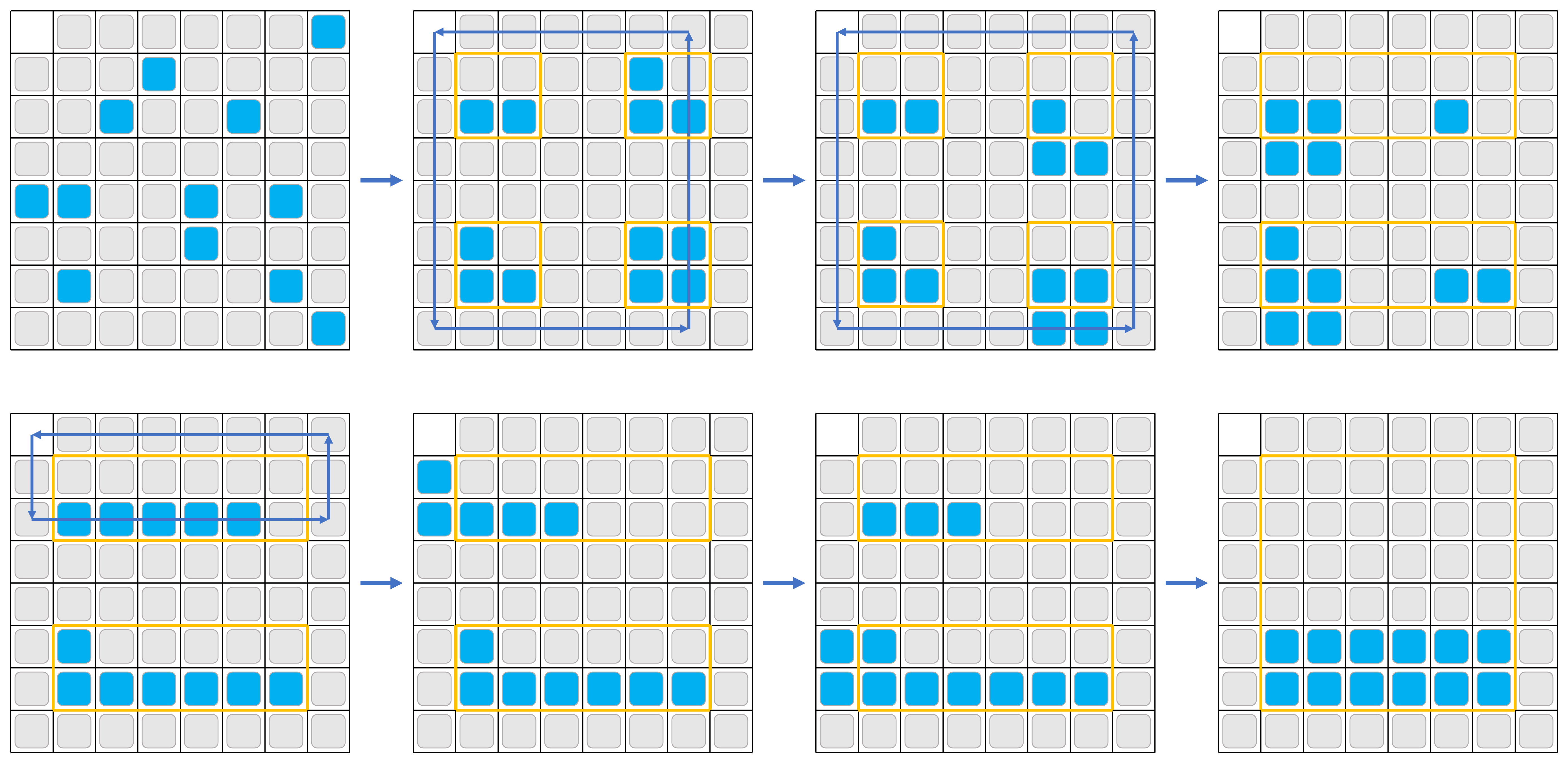}
    \caption{Illustration of the algorithm on an $8 \times 8$ instance, starting with $4 \times 4$ squares.
    From a starting configuration (1), the tiles are moved into the inner squares to reach (2) as in the 2nd stage.
    Then the squares are merged in matching pairs in the horizontal direction from (2)--(5) as in the 3rd stage, where blue tiles in the even column group are brought down to get (3), followed by highway rectangular shifts to move them underneath the matching square as in (4) to be subsequently absorbed as in (5).
    The same procedure is applied in the vertical direction from (5)--(8) to get a nearly sorted configuration, which is then brought down to the bottom row as in the 4th stage.}
    % \label{fig:enter-label}
\end{figure}

In the second stage, we move all the black tiles to the bottom of their respective $8 \times 8$ square left adjusted, manually bringing down each black tile individually in doing so to get a constant cost for each black tile, so this step can be completed also in $O(B)$ steps.
To make sure that the escort doesn't modify the density of black tiles much, one can snake the escort around the grid of $8 \times 8$ squares from the top row to the bottom, moving to the square below after traversing all the squares in a given row.

In the third stage, we now implement our divide-and-conquer algorithm.
Suppose we are on step $i$ of the algorithm, going from $i = 4, \ldots, \log m$.
In the first part, we combine the rows as follows:
for every other block column of $2^{i-1} \times 2^{i-1}$ squares starting with the second, use the escort to drag down each column containing a black tile once in $O(\min(m, B))$ steps.
Then, for each block row with at least one black tile positioned on the \emph{highway}, or the bottom row of each square, use the escort to advance the highway by $2^{i-1}$ steps to the left to be merged with the corresponding square on the left in $O(\min(m/2^i, B)2^i)$ steps.
Then use the escort to absorb the transported black tiles into the destination square in $O(B)$ steps, which takes care of most of the sparse merging pairs of squares.
For the dense ones that contain $c \geq 2^{i-1}$ black tiles, move the remaining black tiles to fill the inner $(2^{i-1} - 2) \times (2^i - 2)$ rectangle in the desired order (bottom row, left-adjusted) in $O(c)$ steps.
One way this can be done is by connecting the two piles of black tiles and then filling the black tiles in the desired order.
Thus, across all the squares, this takes $O(B)$ steps.

After merging along the rows, we merge the rectangles along the columns using the same algorithm.
However, we additionally want to take advantage of the highway that batches black tiles across the grid, so we use the escort to reorient the black tiles in every other block row, starting from the second, so as to fill just the leftmost column if possible, taking $O(B)$ steps.
Then moving the escort as before, instead across the rows to drag the black tiles onto the highway and then columns to drag them downwards takes $O(\min(m/2^i, B) 2^i)$ steps while absorbing the transferred black tiles in $O(B)$ steps takes care of the sparse instances (i.e. $c < 2^{i-1}$) while dense instances can again be combined in $O(c)$ time by connecting the two piles in $O(2^{i-1})$ time and then building the black pile in the desired order in $O(c)$ time for an overall $O(B)$ additional time steps.

Thus in the fourth stage, after completing the divide and conquer stage, the black tiles will be piled near the bottom of the grid except for a unit margin of white tiles.
Then, the escort can be used to drag the black tiles to fill the bottom of the grid, left-adjusted, in $O(\min(m, B))$ steps as needed.
Combining everything, it takes $O(B)$ steps to complete the first, second, and fourth stages, while the divide and conquer algorithm takes $O(m 
\log m + B \log m)$ time steps.
Plugging in $B = \Theta(m^r)$ gives the desired bound.

The case when $m$ is not a power of two or when $m_1, m_2 = \Theta(m)$ can be handled similarly by accounting for border rectangles, but we can more simply choose the largest power of two $2^s$ below both $m_1$ and $m_2$ and run the algorithm on each of a covering set of $2^s \times 2^s$ squares, collecting the black tiles at the bottom afterward to obtain the same time complexity.
~\qed
\end{proof}

\begin{proof}[Proof of \textup{Theorem~\ref{t:mebgspa}}]
We sketch the proof strategy, which is similar to the single escort case, except that we now need to move the escorts efficiently to ensure the desired time complexity is achieved. 
First, the statement when $B/p \leq m$ requires a time complexity of $O(m \log m)$, and so we can forget about the additional escorts and instead show the statement when $B/p \geq m$ as $B/p \approx m$ gives the $O(m \log m)$ time complexity.
    Thus, assume that $r > 1$ and $p < B/m$.
    We again can reduce to the case when $m_1 = m_2 = m$ is a power of two by running the algorithm a few times to line up the black tiles along the bottom of the grid (although not necessarily spread out as needed).
    In order to spread out the black tiles to get the appropriate time complexity, we can utilize each of the $p$ escorts to handle a small portion of the black tiles without interfering with each other (see figure), with each escort handling $B/p > m$ tiles permitting each escort to move their assigned black tiles in $O(B/p + m) = O(B/p)$ time, subsumed by the $O(m^r \log m / p)$ term needed.

    Now, we show that each of the steps of the single escort algorithm can be effectively done in parallel.
    The first stage can be completed by dividing the grid into several row blocks that have about $\Theta(B/p)$ black tiles each, where each escort is assigned one of the row blocks to separate into partial, white, and black rows;
    the escorts can initially be moved to the right column before partitioning the rows due to the possible change of black tiles in each row block as escorts move.
    Afterward, the escorts will be on separate rows so that they can move to the appropriate column in the subsequent column block assignment, thus completing the spreading stage in $O(B/p)$ time.

    The second stage is handled by assigning the escorts connected components of $8 \times 8$ squares, each containing $\Theta(B/p)$ black tiles, which they can all move to in a single time step being on separate columns and without affecting the density of the black tiles significantly.
    Thus, this stage can be completed in $O(B/p)$ time.

    The third stage is the most interesting:
    the default position of the escorts is to be on the bottom row or left column, which they can fit on as $p < B/m < m$.
    When pulling down columns to put black tiles onto highways, we assign each escort $O(\min(m, B) / p)$ columns to handle, which can be done parallel after moving the escorts in $O(m)$ time along the bottom row.
    Then move the escorts to the left column in $O(m)$ steps, assigning each escort to handle $O(\min(m / 2^i, B) / p)$ of the highways to complete the transportation step in $O(\min(m / 2^i, B) 2^i / p)$ time steps.
    Afterwards, assign each of the escorts connected components of $2^i \times 2^i$ squares, each with $O(B/p)$ black tiles in them;
    the escorts can be moved along the margins in $O(m)$ steps to their destination blocks through Rubik Tables as follows:
    since each of the rectangles has at most a 1/2 density of black tiles, the top half of the rectangle (except for possibly a few rows) will contain no black tiles, and so the subgrid consisting of all the $\approx m/2$ rows of white tiles along with the $\Theta(m/2^i)$ column margins contains all the escorts and can be arbitrarily reconfigured in $O(m)$ steps by simulating the Rubik Table algorithm.
    % (cite ref).

    Then in their assigned connected component, the escorts can absorb the black tiles, where sparse merging instances will be part of components with a single escort since $p > B/m$, leading to the completion of those connected components in $O(B/p)$ time.
    For the dense rectangles with more than 1 escort, the connected component will consist of simply the merging pair of rectangles.
    Then these can be effectively merged in $O(B/p)$ time by assigning each escort to handle $O(B/p)$ of the black tiles in much the same way that the black tiles are spread out along the bottom row in our initial reduction to the square grid case.
    Afterward, escorts can be transported back to the bottom row in $O(m)$ steps via the underlying Rubik table of white tiles before being moved to the left column to simulate everything again.
    The only difference is that escorts must mediate between the bottom row to be transported to and from their destination spot in the grid in $O(m)$ steps, ensuring a faithful simulation of the divide and conquer approach in parallel.

    The last stage can simply be done in $O(\min(m, B))$ time as with a single escort.
    Thus, the overall running time will be $O(m \log m + B \log m / p)$, or $O(m \log m + m^r \log m / p)$.
~\qed
\end{proof}

\begin{proof}[Proof of \textup{Theorem~\ref{t:hdbgspa}}]
We again sketch a similar approach, highlighting the differences. 
The divide and conquer algorithm more or less extends to the case with higher dimensions by design.
    We highlight the main differences below.
    Again, note that the statement reduces to the case in which $B/p \geq m$, i.e. $r > 1$ and $p < B/m$, where we work on an $m \times \cdots \times m$ grid for $m$ a power of two, gaining a potential constant factor as a function of $d$ (which is viewed as a constant).
    This means that each escort is assigned at least $m$ associated black tiles, so all of them fit in a $d-1$ dimensional hyperplane e.g. any single border of the $d$ dimensional grid, important to ensuring that additional escorts add to the efficiency of parallelization of the algorithm.
    In our reduction, we again need to spread out the black tiles along the bottom border of the grid. By efficiently assigning each escort to move $O(B/p)$ black tiles a distance of potentially $\Omega(m)$, we can again batch movement and complete the reduction in $O(B/p + m) = O(B/p)$ time, subsumed by the $O(m^r \log m / p)$ term again.

    As in the previous algorithm, we will again need to utilize the higher dimensional Rubik Tables to ensure that escorts can be moved quickly throughout the grid to simulate all that a single escort could do efficiently.
    Specifically, we can arbitrarily move the escorts in a given $d$ dimensional grid, in $O(n)$ steps (treating $d$ as a constant) assuming that all the occupied tiles are of the same color.

    In the first stage of the algorithm, we first move the escorts to $(1, -, \ldots, -)$ in $O(m)$ steps.
    We now consider subgrids of the form $(-, -, x_3, \cdots, x_n)$ separately for $1 \leq x_3, \ldots, x_n \leq m$, splitting each 2D grid into row blocks as before and assigning each escort to handle $O(B/p)$ black tiles.
    We then reconfigure the escorts along the border in $O(m)$ steps to reach their assigned 2D grid and move to their row block, completing the spreading process in $O(B/p)$ steps, moving them back to the border in another $O(m)$ steps.

    For the second stage, we split up the grid into boxes of size $2^s \times \cdots \times 2^s$, choosing a rougher level of granularity to ensure that the inner box containing the black tiles contains at least 9/10 of the volume of the whole box;
    this is important in ensuring that most of the upper part of the inner box has white tiles to allow for the escort reconfiguration process in the divide and conquer algorithm.
    Note that $s$ is a function of $d$, which is viewed as a constant.
    Then escorts are reconfigured in the border in $O(m)$ steps to jump to their assigned connected component of boxes each containing $O(B/p)$ black tiles in a single step, preserving the density of black tiles throughout.
    Then the black tiles can be sorted to the bottom of their inner box (filling in the first coordinate, then the second, etc.), from which escorts can be rerouted along the subgrid of white tiles through the higher dimensional Rubik Table algorithm into the top of the grid $(-, \ldots, -, m)$ in $O(m)$ steps.

    The third stage happens in the usual way, except that we now merge along each of the $d$ coordinates in order, where the highway for the first $d-1$ coordinates will lie in the margins below the inner boxes in the direction of the current dimension;
    the highway for the last dimension can be chosen to be along the first dimension.
    Then for a given dimension, we may need to modify the black tiles slightly to ensure maximal use of the highway transportation, assigning connected components with $O(B/p)$ black tiles each, dispersing the escorts in $O(m)$ steps from $(-, \ldots, -, m)$, reorienting the black tiles in $O(B/p)$ time, and sending the escorts back to the top of the grid in $O(m)$ time.
    Then, the escorts can be moved along the borders from $(-, \ldots, -, m)$ to the desired border, assigning each escort $O(\min(m^{d-1}, B) / p)$ columns to drag down black tiles onto the highway, which can be done simultaneously in $O(\min(m^{d-1}, B) / p)$ time steps after a reconfiguration in the border in $O(m)$ steps.
    Afterward, they can be moved along the borders in the direction of the highway and reconfigured in $O(m)$ steps, assigning each escort to handle $O(\min([m/2^i]^{d-1}, B) / p)$ highways, which they can complete in $O(\min([m/2^i]^{d-1}, B) 2^i / p)$ time steps.
    They can then be moved back to the border and to the top of the grid, reconfigured, and then dispersed through the grid in $O(m)$ steps to the assigned blocks containing $O(B/p)$ black tiles.
    Sparse instances are again completed in $O(B / p)$ time steps, while dense instances consist of a single pair of boxes to be joined with potentially several escorts, all of which fit inside the top margin of the pair of boxes as we assumed that $B/p \geq m$.
    Thus, dense blocks can be solved in $O(B/p + 2^i)$ or $O(B/p)$ time by effectively splitting up the black tiles, $O(B/p)$ to be handled by each escort such as in the algorithm to spread the black tiles along the bottom of the grid in the initial reduction.
    Then escorts can be moved back to the bottom of the grid again in $O(m)$ steps, repeating this process for each dimension and level of granularity until we are left with a single box.

    Finally, we can complete the last stage by splitting up the escorts to handle $O(B/p)$ black tiles again, moving down the black tiles in $O(B/p)$ time steps.
    Thus, the overall instance takes $O(m \log m + B \log m / p)$ time steps as well as the additional $O(\sum_i \min(m^{d-1}, B) / p + \min([m/2^i]^{d-1}, B)2^i / p)$ time steps.
    When $r < d-1$, the first term becomes $O(B \log m / p)$, whereas the second term dominates when $B \approx[m/2^i]^{d-1}$, which rearranges to $2^i \approx m^{1 - r/(d-1)}$, giving $O(m^{1 + r [\frac{d-2}{d-1}]}/p)$.
    This gives the desired bound when $r < d-1$.
    For $r \geq d-1$, the second term remains the same while the first term becomes $O(m^{d-1} \log m / p)$, which is subsumed by the $O(B/p) = O(m^r \log m / p)$ time.
    Thus, we again obtain the desired $O(m \log m + m^{1 + r[\frac{d-2}{d-1}]} / p + m^r \log m / p)$ time complexity.
~\qed  
\end{proof}

\begin{proof}[Proof of \textup{Theorem~\ref{t:lb3}}]
    The underlying motivation is in the difficulty of routing tiles at the critical level of granularity as in the analysis of the upper bound, which is when $2^i \approx m^{1 - r/(d-1)}$; denote this quantity as $\ell$.
    Consider a random placement of the black tiles under this level of granularity.
    One of the $2^d$ sub-lattices $\ell$ obtained from the grid of $\ell \times \cdots \times \ell$ boxes of wire length 2 (wire length $2\ell$ in the overarching grid) contains a maximal number of black tiles with at least $1/2^d$ of them.
    Then we can lower bound this instance by the grid collapse game only on $\ell$, ignoring all of the other black tiles.
    Note that in any solution, a black tile must touch one of the $d-1$ dimensional hyperplanes passing through any of the centers of the other $\ell \times \cdots \times \ell$ boxes that aren't a part of $\ell$ (where we snap the destination position to this lattice and lose at most an additive term of $\ell \ll m^{1 + r[\frac{d-2}{d-1}]}$), and so we can lower bound by the problem of moving a black tile to touch at least one of these hyperplanes.
    Furthermore, this is lower bounded by the problem in which boxes containing a black tile must simply use a generalized column intersecting the region it occupies cut out by the hyperplanes, at least $\ell/2$ times (the gap between the box and each of the surrounding hyperplanes).
    Note that by randomness, we can assume that there are $B/2^d$ black tiles randomly taken from $\ell$, and furthermore, black tiles are chosen with replacement (since to get without replacement, we simply reroll more black tiles to get more black tiles and thus harder instances).

    To summarize, the problem has reduced to the following:
    consider the $(m/\ell) \times \cdots \times (m/\ell) = m^{r/(d-1)} \times \cdots \times m^{r/(d-1)}$ grid, where we randomly sample $B' = B/2^d = \Theta(m^r)$ black tiles with replacement; let $B' = c \cdot m^r$ for a stable constant $c$.
    Then we need to choose a sequence of columns so that each cell of this grid (representing an $\ell \times \cdots \times \ell$ box in $\ell$) touches at least $\ell/2$ of the columns.
    Then due to the relative sparsity of the grid, we can use a balls and bins argument to show that with high probability, columns contain relatively little black tiles, requiring us to use a lot of column operations.

    Consider the columns along one of the $d$ dimensions, and let $X_p$ be the number of columns in that direction that contain exactly $p$ of the black tiles (that may potentially occupy the same cell of the grid), so $X_p = \sum_{i=1}^{m^r} X_{pi}$ for $X_{pi}$ indicating whether the $i$th column contains exactly $p$ black tiles.
    Then $$\E[X_p] = m^r \E[X_{pi}] = m^r \binom{B'}{p} \left(\frac{1}{m^r}\right)^p \left(1 - \frac{1}{m^r} \right)^{B' - p} \leq m^r \frac{(B')^p}{p!} \cdot \frac{1}{m^{rp}} = B' c^{p-1} / p!$$
    and $$\Var[X_p] = m^r \Var[X_{pi}] + m^r (m^r - 1) \Cov(X_{pi}, X_{pj}) \leq m^r \Var[X_{pi}]$$
    where we have used that $\Cov(X_{pi}, X_{pj})$ is negatively correlated (knowing that a column contains $p$ black tiles means that there are much less black tiles spread around the other columns);
    this can be manually checked since 
    
    $$\E[X_{pi} X_{pj}] - E[X_{pi}]^2 \leq 0$$
    if and only if 
    $$\binom{B' - p}{p} \left(\frac{1}{m^r - 1}\right)^p \left(1 - \frac{1}{m^r - 1} \right)^{B' - 2p} \leq \binom{B'}{p} \left( \frac{1}{m^r} \right)^p \left(1 - \frac{1}{m^r} \right)^{B' - p},$$
    which is true since $$\frac{1}{m^r - 1} \cdot \frac{m^r - 2}{m^r - 1} \leq \frac{1}{m^r}.$$
    Then since $\Var[X_{pi}] \leq \E[X_{pi}]$, Chebyshev's bound gives that $$\Pr(|X_p - \E[X_p]| \geq t) \leq \Var(X_p)/t^2 \leq m^r\E[X_{pi}] / t^2 \leq B' c^{p-1} / [p! t^2],$$
    and so setting $t = \frac{ 
 p c^{p-1} B'}{\sqrt{p!}}$ gives $$\Pr( |X_p - \E[X_p]| \geq c^{p-1} B' / \sqrt{p!}) \leq 1 / (p^2 B' c^{p-1}).$$
 Thus, by union bounding over all $1 \leq p \leq m^r$, with a probability of failure of at most $\pi^2 / 6B' c^{p-1}$ i.e. high probability, we have that $X_p \leq \E[X_p] + t \leq 2t \leq 2pc^{p-1} B' / \sqrt{p!}$.
 By union bounding over the $d$ dimensions, we can assume that the total number of columns that have exactly $p$ black tiles $Y_p$ is at most $2pdc^{p-1} B' / \sqrt{p!}$, which happens with a probability of failure of at most $d\pi^2 / (6 B' c^{p-1}) = o(1)$.

 Now, the best possible solution would try to maximize the number of columns that contain a lot of black tiles, possibly with difficulties due to overlap between the $d$ dimensions.
 Nevertheless, we can lower bound by the case in which $X_p = \lfloor 2pc^{p-1} B' / \sqrt{p!} \rfloor$ and there is no interference among the columns.
 Thus, an optimal solution would then go through the list of columns in reverse order of the number of black tiles, using each a maximal number of $\ell/2$ times (since each black tile only needs to be hit $\ell/2$ times).
 Thus, we seek a maximal $S$ such that $$\sum_{p = S}^\infty d (\ell/2) p \lfloor 2pc^{p-1} B' / \sqrt{p!} \rfloor  \geq (\ell/2)B'',$$
 where $B''$ is the number of distinct black tiles chosen;
 this is with high probability at least $f \cdot B'$ for a constant $f$ as we select $B'$ random black tiles from $m^{rd / (d-1)} \gg B'$ (i.e. using the standard balls and bins fact that throwing $n$ balls into $n$ bins covers a positive fraction with high probability).
 Then we can relax the condition on $S$ to drop the floor function and instead seek $S$ satisfying $$\sum_{p=S}^\infty c^{p-1}\frac{p^2}{\sqrt{p!}} \geq \frac{f}{2d}.$$
 The LHS eventually dies off faster  than exponentially due to the $1 / \sqrt{p!}$ factor, so the main issue is choosing an $S$ such that $c^{S-1} S / \sqrt{S!} = \Omega(1 / d)$, which requires $S = \Theta(\log d)$.
 Thus, such an $S$ for bounding an optimal solution (since the sum forces us to pick a column and use it $\Theta(d\ell)$ times) loses an additive factor of at most $O(d\ell)$, where $d\ell \ll m^{1 + r[\frac{d-2}{d-1}]}$, which doesn't affect the target lower bound.
 Then the total number of moves used will be $$\sum_{p = S}^\infty d (\ell/2) \lfloor 2pc^{p-1} B' / \sqrt{p!} \rfloor$$
 which we seek to show is $\Omega(\ell B')$.
 Indeed, at most $\Theta(\log B')$ of the terms on the LHS will be nonzero, so it suffices to show that $$\sum_{p=S}^\infty d(\ell/2)2pc^{p-1} B' / \sqrt{p!} - (d\ell/2) \Theta(\log B') = \Omega(lB'),$$
 reducing to simply $$\sum_{p=S}^\infty d\ell p c^{p-1} B' / \sqrt{p!} = \Omega(\ell B').$$
 This reduces to showing that $$\sum_{p=S}^\infty p / \sqrt{p!} = \Omega(1 / d),$$
 which is the case as $S = \Theta( \log d)$ (satisfying nearly the same equation except with $p^2$ instead of $p$ and with a higher constant on the RHS) where the $\sqrt{p!}$ factor is the dominating term.
 Thus, we have an $\Omega(\ell B') = \Omega(m^{1 - r / (d-1)} m^r) = \Omega(m^{1 + r[\frac{d-2}{d-1}]})$ high probability lower bound as desired.
 ~\qed
\end{proof}
\end{document}